\documentclass[]{article}
\usepackage{geometry}
\usepackage[round]{natbib}
\usepackage{amssymb}

\usepackage{amsmath}
\usepackage{amssymb}
\usepackage{amsthm}
\usepackage{bm}
\usepackage{bbm}
\usepackage{graphicx}
\usepackage{algorithm}
\usepackage{algpseudocode}
\usepackage{optidef}
\usepackage{caption}
\usepackage{subcaption}
\usepackage{color}
\usepackage{enumitem}

\newtheorem{theorem}{Theorem}
\newtheorem{lemma}[theorem]{Lemma}
\newtheorem{corollary}[theorem]{Corollary}

\newcommand{\X}{\mathcal X}
\newcommand{\Y}{\mathcal Y}
\newcommand{\w}{w}
\newcommand{\dist}{d}
\newcommand{\distw}{\tilde \dist}
\newcommand{\bX}{\boldsymbol X}

\newcommand{\bY}{\bm Y}

\newcommand{\err}{\textup{err}}
\newcommand{\tcl}{h}
\newcommand{\cl}{h}

\newcommand{\chr}[1]{\mathbbm{1}_{\{#1\}}}
\newcommand{\emperr}{\widehat{\textup{err}}}
\DeclareMathOperator{\argmax}{\textup{argmax}}
\DeclareMathOperator{\argmin}{\textup{argmin}}
\newcommand{\ds}{S}
\newcommand{\dsc}{\tilde \ds}
\newcommand{\dsp}{\ds^+}
\newcommand{\dsn}{\ds^-}
\newcommand{\dscp}{\dsc^+}
\newcommand{\dscn}{\dsc^-}
\newcommand{\dsco}{\dsc^*}

\newcommand{\tX}{\tilde x}
\newcommand{\tY}{\tilde y}

\newcommand{\PROB}{\mathbb{P}}
\newcommand{\EXP}{\mathbb{E}}
\newcommand{\eps}{\varepsilon}
\newcommand{\net}{\boldsymbol{X}}

\newcommand{\supp}{\textup{supp}}

\newcommand{\unifset}{U}
\newcommand{\ball}{B}

\newcommand{\wset}{\tilde{\mathcal W}}

\newcommand{\rad}{r}
\newcommand{\rado}{\rad^*}

\newcommand{\recon}{\mathcal R}
\newcommand{\I}{\mathcal I}

\DeclareMathOperator{\ddim}{ddim}
\DeclareMathOperator{\diam}{diam}
\DeclareMathOperator{\enemy}{ne}
\DeclareMathOperator*{\argminn}{arg\,min}

\begin{document}

%

%

%
%
%

\title{Weighted Distance Nearest Neighbor Condensing}

\author{Lee-Ad Gottlieb \and Timor Sharabi \and Roi Weiss}

\date{Ariel University}

\maketitle 

\begin{abstract}
The problem of nearest neighbor condensing has enjoyed a long history of study, both in its theoretical and practical aspects. In this paper, we introduce the problem of weighted distance nearest neighbor condensing, where one assigns weights to each point of the condensed set, and then new points are labeled based on their weighted distance nearest neighbor in the condensed set. 

We study the theoretical properties of this new model, and show that it can produce dramatically better condensing than the standard nearest neighbor rule, yet is characterized by generalization bounds almost identical to the latter. 
We then suggest a condensing heuristic for our new problem. We demonstrate Bayes consistency for this heuristic, and also show promising empirical results.
\end{abstract}

\section{Introduction} 

The nearest neighbor (NN) classifier, introduced by \cite{fixandhodge}, is an intuitive and popular learning tool. In this model, a learner observes a sample of (typically binary) labeled points, and given a new point to be classified, assigns the new point the same label as its nearest neighbor in the sample.
%
%
It was subsequently demonstrated by \citet{1053964} that when no label noise is present, the nearest neighbor classifier's expected error converges to zero as the sample size grows.
This and other results helped spawn a deep body of research into proximity-based classification  \citep{devroye1996probabilistic,shwartz2014understanding}.

Nearest neighbor classifiers enjoy other advantages as well. They require only a distance function on the points, and do not even require that the host space be metric. They also extend naturally to the multi-class setting. Yet they are not without their disadvantages: For example, a naive nearest neighbor approach may require storing the entire sample. 
To address the disadvantages of the NN classifier, \cite{hart1968condensed} introduced the technique of sample compression for the NN classifier. This work defined the minimum consistent subset problem (also called the nearest neighbor condensing problem): Given a sample, find a minimal subset of the sample that is {\em consistent} with it, meaning that for every point of the sample, its nearest neighbor in the subset (that is the condensed set) has the same label. \cite{hart1968condensed} further suggested a heuristic for the NN condensing problem. Analysis of this problem, as well as the creation of new heuristics for it, has been the subject of extensive research since its introduction \citep{gates72,RWLI-75,devroye1996probabilistic,wilson00}.

\paragraph{Our Contribution}
In this paper, we introduce a novel modification of the NN condensing problem. In our new condensing problem (presented formally in Section \ref{sec:rules}), each point of the condensed set is also assigned a weight. We modify the distance function to consider \emph{weighted distance}, meaning that the distance from a new point to a point in the condensed set is their original distance divided by the weight of the point in the condensed set. It follows that assigning high weight to a point in the condensed set increases its influence on the labeling of new points. We call the new problem of choosing a condensed set and assigning its weights the \emph{weighted distance nearest neighbor condensing problem}.

We proceed with an in-depth study of the statistical properties of weighted distance condensing (Section \ref{sec:properties}). Crucially, we find that our model allows dramatically better condensing than what is possible under standard (that is, unweighted) NN condensing: There are families of instances wherein the optimal condensing under unweighted NN rule is of size $\Theta(n)$, while condensing under the weighted NN (WNN) rule yields a condensed set of size exactly 2 (Theorem \ref{thm:condensing}). At the same time, we demonstrate generalization bounds for condensing under the WNN rule which are almost identical to those previously known for the NN rule (Theorem \ref{thm:generalization bound}). This means that the much more powerful weighted rule can be adopted with only negligible increase in the variance of the model, so that the more powerful rule does not contribute to overfitting.

Having established these statistical properties of WNN condensing, we suggest a greedy-based heuristic for this problem, that is the identification of a condensed set and the assignment of weights to members of this set (Section \ref{sec:heuristic}). We further demonstrate that the suggested heuristic is a member of a broad set of classifiers which are Bayes consistent, thereby lending statistical support to its use.

After deciding on a heuristic for our condensing problem, we compare its empirical condensing abilities to those of popular heuristics for the unweighted NN problem (Section \ref{sec:results}). We find that that the condensing bounds of our heuristic compare favorably to the others, and this suggests that further research on heuristics for WNN condensing is a promising direction.

Finally, in Section \ref{sec:search}, we consider the related problem of finding a nearest neighbor for a query point, under weighted distances -- this is exactly what must be done to label a new point under WNN condensing. We give a novel weighted distance approximate nearest neighbor search algorithm based on the navigating nets of \cite{KL04}.

\subsection {Related Work}

The nearest neighbor condensing problem is known to be NP-hard \citep{Wil-91,Zuk-10}, and \cite{hart1968condensed} provided the first heuristic for it. Many other heuristics have been suggested since, and we mention only a few of them here: 
\cite{gates72} introduced the reduced nearest neighbor (RNN) rule heuristic to  iteratively contract the sample set.
\cite{RWLI-75} introduced the selective subset heuristic, which additionally guarantees that for any sample point, the distance to its nearest neighbor in the compressed set is less than the distance to any sample point with opposite label. \cite{BFS-05} subsequently suggested a modification to this algorithm, which they called modified selected subset (MSS).
\citet{angiulli2005fast} introduced the fast nearest neighbor condensing (FCNN) heuristic, while \citet{flores2020social} introduced the RSS and VSS heuristics, which modify the FCNN algorithm to improve its behavior in cases where the points are too close to each other.
Another popular heuristic, modified condensed nearest neighbor (MCNN), was introduced by  \cite{devi2002incremental}. 

 No concrete algorithmic condensing bounds were known for NN condensing until the work of \citet{gottlieb2018near}: They derived hardness-of-approximation results, and designed an algorithm called NET, which computes in polynomial time an approximation to the minimum subset almost matching the hardness bounds. This approach was later extended by \citet{gottlieb2019classification} to asymmetric distance function. Also related to this is the result of  \citet{gottlieb2020non}, which presented non-uniform packing, that is using balls with different radii.

As for the statistical properties of NN condensing rules, \citet[Chapter 19]{devroye1996probabilistic} established Bayes consistency for an intractable rule that searches for a condensed set of fixed, data-independent size, minimizing the empirical error on the entire sample. They showed that universal Bayes consistency is achieved in finite-dimensional spaces provided the size of the condensed set is sub-linear in the sample size. 
\citet{HKSW} introduced a computationally-efficient data-dependent sample-compression NN rule, termed OptiNet, that computes a net of the samples and associates each point of the condensed set with the majority vote label in its Voronoi cell. They showed that OptiNet is universally Bayes consistency in any separable metric space. A simpler sub-sampling NN rule achieving universal Bayes consistency in separable metric spaces was demonstrated by \citet{gyorfi2021universal}, establishing also error rates. 
\citet{xue2018achieving} studied the error rates achieved by more complex sub-sampling NN rules. Lastly,  
\citet{kerem2023error} studied jointly-achievable error and compression rates for OptiNet under a margin condition.

\subsection{Preliminaries}

\paragraph{Metric Space.} A \emph{metric} $d$ defined on a point set $\X$ is a positive symmetric function satisfying the triangle inequality, i.e.\ $d(x,y) \le d(x,z) + d(z,y)$. The set $\X$ and metric $d$ together define the metric space $(\X,d)$. 
Let $B(x,r)$ denote the (open) ball centered at $x$ with radius $r$; a point $y \in \X$ is in $B(x,r)$ if $d(x,y) < r$.

\paragraph{Notation.} We use $[k]$ to denote $\{1,\ldots,k\}$. We define the distance between a point $y$ and set $S$ as the distance of $y$ to the closest point in $S$, that is $d(x,S) = \min_{y \in S} d(x,y)$. For a set $S$, we denote its cardinality by $|S|$.

\section{Nearest neighbor rules}\label{sec:rules}

Given a metric space $(\X,d)$ and a labeled sample $\ds\subset\X$ (where $l(x)$ is the label of point $x \in \X$), the nearest neighbor condensing problem is to find a subset $\dsc\subset\ds$ of minimal cardinality, such that the nearest neighbor rule on $\dsc$ classifies all sample points in $\ds$ correctly;
that is, for any point $x\in\ds$,
$l(x) = l(\displaystyle\argmin_{y\in \dsc} \left(d\left(x,y\right)\right))$.

Now let $\w:\X\to (0,\infty)$ be a positive weight function, and define the weighted distance
\begin{align*}
    \distw(x,x') = \frac{d(x,x')}{w(x) \cdot w(x')}, \qquad x,x'\in\X.
\end{align*}
This is our weighted distance nearest neighbor (WNN) rule,
under which we may define a WNN classifier:
\begin{align*}
    h_{(\dsc,\w)}(x) = l(\arg\min_{y \in \tilde{S}} \distw(y,x)),\qquad x\in\X.
\end{align*}
Note that the weighted distance may not satisfy the triangle inequality.
In the weighted condensing problem, we seek a subset $\Tilde{S} \subset S$ and a weight assignment $w$ for $\tilde{S}$ 
(with $w(x)=1$ for all $x\in\X \setminus \tilde{S}$), such that for each point $x \in \ds$, the weighted distance nearest neighbor of $x$ in $\dsc$ has the same label as $x$,
$h_{(\dsc,\w)}(x) = l(x)$.

It is easy to see that WNN is a generalization of the NN rule, as the latter can be recovered by simply taking all weights to be equal to 1. Let us motivate our weighted distance function by illustrating the effect of weighting on the Euclidean decision boundary between two points.
Consider two points in the Euclidean plane, 
$p_1=(x_1,y_1)$ and $p_2=(x_2,y_2)$. If $w(p_1)=w(p_2)$, then the WNN decision boundary is equivalent to the NN decision boundary, defined by 
$$\sqrt{(x_1-x)^2+(y_1-y)^2} = \sqrt{(x_2-x)^2+(y_2-y)^2},$$
which can be simplified to a line in slope-intercept form:
\begin{equation*}
y = x \cdot \frac{x_1-x_2}{y_2-y_1} + \frac{-x_1^2 + x_2^2 - y_1^2 + y_2^2}{2(y_2-y_1)}.
\end{equation*}
In contrast, when $w(p_1) \ne w(p_2)$, the decision boundary is defined by 
\begin{equation*}
\frac{\sqrt{(x_1-x)^2+(y_1-y)^2}}{w_1} = \frac{\sqrt{(x_2-x)^2+(y_2-y)^2}}{w_2}.
\end{equation*}
This can be simplified to the standard equation of a circle
\begin{eqnarray*}
\left(x+\frac{w_{1}^{2}x_{2}-w_{2}^{2}x_{1}}{w_{2}^{2}-w_{1}^{2}} 
\right)^2
+\left(y+\frac{w_{1}^{2}y_{2}-w_{2}^{2}y_{1}}{w_{2}^{2}-w_{1}^2}\right)^{2}
& = &
\frac{w_{1}^{2}y_{2}^{2}+w_{1}^{2}x_{2}^{2}-w_{2}^{2}y_{1}^{2}-w_{2}^{2}x_{1}^{2}}
{w_{2}^{2}-w_{1}^{2}}
\\
& & 
+\left( \frac{w_{1}^{2}x_{2}-w_{2}^{2}x_{1}}{w_{2}^{2}-w_{1}^2}\right)^2
+\left(\frac{w_{1}^{2}y_{2}-w_{2}^{2}y_{1}}{w_{2}^{2}-w_{1}^{2}}\right)^2.
\end{eqnarray*}
See Figure \ref{fig:nn-wnn}. The fact that WNN induces a circular boundary (and in higher dimension, a ball separator) will be useful for our proofs and constructions.

\begin{figure}
\centering
\begin{subfigure}{0.5\linewidth}
\centering
\includegraphics[width=.9\linewidth]{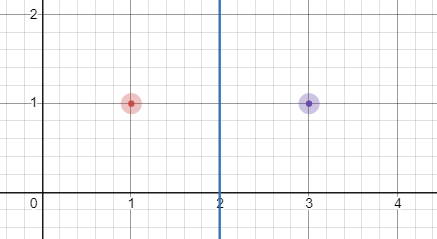}
\end{subfigure}%
\begin{subfigure}{0.5\linewidth}
\centering
\includegraphics[width=.9\linewidth]{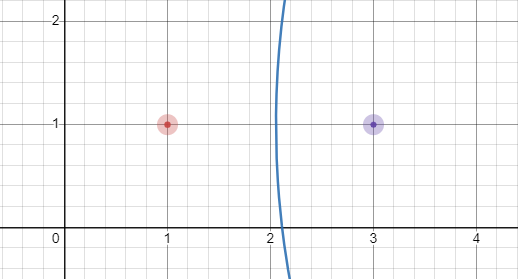}
\end{subfigure}
 \caption{Illustration of decision boundary for equal (left) and unequal weights (right). In the right figure, $w(p_1)=1.5$ and $w(p_2) = 1$, where $p_1$ is its left point and $p_2$ is its right point.}
\label{fig:nn-wnn}
\end{figure}

While the distance function for nearest neighbor condensing rules is often taken to be the Euclidean norm, all our results below hold for all metric distance functions.

\section{Properties of weighted distance nearest neighbor}\label{sec:properties}

In this section, we establish condensing properties under WNN, especially in comparison to condensing properties already known for the regular unweighted nearest neighbor  (NN) rule. We show that the former rule is strictly more powerful: It can always yield condensing as good as the latter, and in some cases better by a factor of $\Theta(n)$. At the same time, we derive generalization bounds for WNN condensing that are essentially the same as those known for NN condensing, so that the utilization of a more powerful tool does not lead to an increase in generalization error.

\subsection{Condensing bounds}\label{sec:condensing-bounds}

As weighted nearest neighbor generalizes unweighted nearest neighbor, it can only improve the condensing. We can in fact show the following:

\begin{theorem}\label{thm:condensing}
For any $n$, there exists an $n$-point data-set that can be condensed to 2 candidates under the WNN rule, but requires $\Theta(n)$ candidates under the NN rule.
\end{theorem}

In place of simply presenting a construction achieving the bound of Theorem \ref{thm:condensing}, we will instead present a new condensing rule we call the ball cover (BC) rule, and show that WNN generalizes this rule as well. We then demonstrate that condensing under NN and BC is incomparable, in that each rule admits $n$-point sets that it can condense to a constant size, but which the other cannot condense below $\Theta(n)$ points. This expanded explanation will yield a better understanding of the power of WNN, and motivate the heuristic of Section \ref{sec:heuristic} below.

\paragraph{Ball Cover Rule.}
We define another rule for condensing, the ball cover (BC) rule: Given the input points $S \subset \X$, we must produce a subset $\Tilde{S} \subset \X$ of minimum cardinality, and also assign each point $x_i \in \Tilde{S}$ a radius $r_i$. Let $B_i = B(x_i, r_i)$. We require that for any two points $x_i,x_j \in \Tilde{S}$ with $l(x_i) \ne l(x_j)$, that $r_i,r_j < d(x_i,x_j)$. It follows that no ball contains sample points of the opposite label. The decision rule simply assigns a point $x \in S$ the same label as the center of a ball containing $x$. A valid BC condensing satisfies that every $x \in S \setminus \Tilde{S}$ is found in a ball $B_i$ satisfying $l(x) = l(x_i)$. (There is however a caveat relating to the labeling of points not in the sample: These points can fall into multiple balls, or no ball at all. In these cases, one may impose some arbitrary decision rule, such as a priority over the balls.)

The BC rule is motivated by the NET algorithm of \citet{gottlieb2018near}. This algorithm can be viewed as covering the space with balls of equal radius. The BC rule is an extension of the NET approach to balls of different radii.

It is easy to show that WNN generalizes the BC rule: Given a condensed set $\Tilde{S}$ with assigned radii $r_i$, a WNN classifier can be produced by taking the same points of $\Tilde{S}$, and assigning weight $w(x_i) = r_i$ for all $x_i \in \Tilde{S}$. Consider any point $x \in S \setminus \Tilde{S}$ falling in $B_i$ but not in $B_j$, and we have that 
$ \distw(x,x_i)  = \frac{d(x,x_i)}{w(x_i)}
 < \frac{r_i}{r_i}
 = 1,$
while
$ \distw(x,x_j)  = \frac{d(x,x_j)}{w(x_j)}
 \ge \frac{r_j}{r_j}
 = 1,$ 
 so that the WNN rule is indeed consistent with the BC rule on the sample.

\paragraph{Comparison between NN and BC.}
It remains to prove the following lemma, concerning condensing under the NN and BC rules:

\begin{lemma}\label{lem:NN-BC}
For any $n$, there exists an $n$-point data-set that can be condensed to 2 candidates under the NN rule, but requires $\Theta(n)$ candidates under the BC rule.

Likewise, for any $n$, there exists an $n$-point data-set that can be condensed to 2 candidates under the BC rule, but requires $\Theta(n)$ candidates under the NN rule.
\end{lemma}

As WNN generalizes BC, Theorem \ref{thm:condensing} follows immediately from Lemma \ref{lem:NN-BC}. It remains to prove the lemma.

\begin{proof}[Proof of Lemma \ref{lem:NN-BC}]
For the first item of the lemma,
assume without loss of generality that $n$ is even, and set $\gamma = n/2$.
Our example set $\X$ will have $\gamma$ red points and $\gamma$ blue points, with a diameter $\Theta(n)$ and minimum distance of $1$ between the points. Let $c = \frac{1}{2}$, and add to the set $\gamma$ red points at positions 
$\{(i,c) : i \in [\gamma]\}$, 
and $\gamma$ blue points at positions
$\{(i,-c) : i \in [\gamma]\}$.
See Figure \ref{fig:NN-BC}.

Now under the NN rule, there exists a solution which uses only two points: This solution $S$ takes the red point at $(1,c)$ together with the blue point at $(1,-c )$: Take any red point at $(i,c)$, and its distance from the red candidate is $i-1$, while its distance from the blue candidate is exactly
$\sqrt{(i-1)^2 + 4c^2} > i-1$. A similar argument applies to the blue points, and so we conclude that $S$ is a consistent condensing of $\X$.

Turning to the BC rule, we show that any solution under this rule must have all $n$ points: Assume by contradiction that we can consistently cover all the points of $\X$ with only $k$ balls, where $k<n$. This implies that there exists some solution ball which covers 2 or more points of the same color; assume without loss of generality that this ball is red, and its center is $p_i = (i,c)$. As this ball covers more than a single red point, its radius must be greater than $1$, but then it contains the blue point $(i, -c)$, which is forbidden. We conclude that no ball contains more than one point, and so the optimal solution under BC must contain exactly $n$ points.

\begin{figure}
\centering
\begin{subfigure}{0.5\linewidth}
\centering
\includegraphics[width=0.9\linewidth]{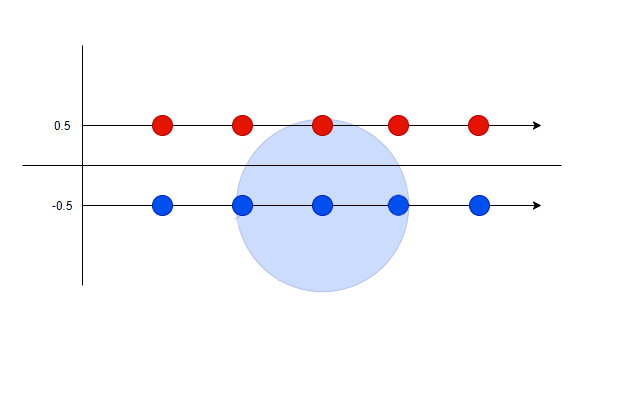}
\end{subfigure}%
\begin{subfigure}{0.5\linewidth}
\centering
\includegraphics[width=0.9\linewidth]{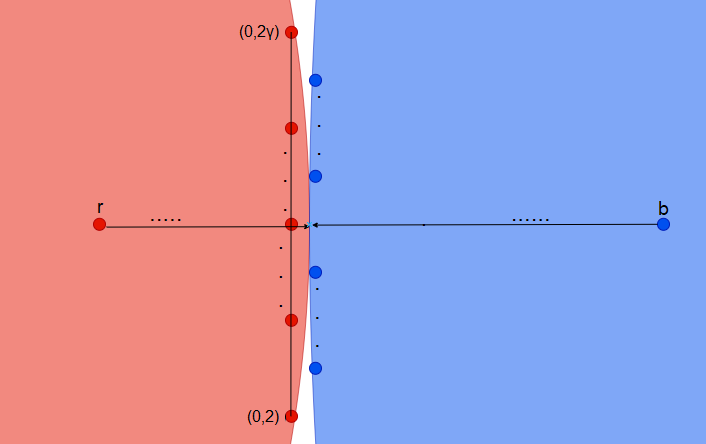}
\end{subfigure}
 \caption{Left: A set that admits good condensing under the NN rule, but not under the BC rule. Right: A set that admits good condensing under the BC rule, but not under the NN rule.}
\label{fig:NN-BC}
\end{figure}

For the second item of the lemma,
assume without loss of generality that $n$ is odd, set $\gamma = \frac{n-1}{2}$, and further assume that $\gamma$ is odd. Our example has $\gamma$ blue points and $\gamma + 1$ red points: Let the points $\{(0,2i) : i \in [\gamma]\}$ be red, and the points $\{(1,2i + 1) : i \in [\gamma - 1]\}$ be blue. Set $t=\frac{(\gamma+1)^2}{2}$, and add an additional red point  $r = (-t,\gamma+1)$, and an additional blue point $b = (2t,\gamma+1)$. See Figure \ref{fig:NN-BC}.

We show that under the BC rule, two balls are sufficient to correctly label all points: First take the red ball centered at $r$ with radius 
$\sqrt{t^2 + (\gamma + 1)^2} = \sqrt{t^2 + 2t} < t+1$. Clearly this ball contains no blue points, since the line containing the blue points is at distance exactly $t+1$ from $r$. But the ball contains all the red points, as the farthest red points from $r$ are at $(0,2)$ and $(0,2\gamma)$, both of which are at distance $\sqrt{t^2 + (\gamma-1)^2}$ from the ball center, and hence inside the ball.
Now take the blue ball centered at $b$ with radius $\sqrt{4t^2 + (\gamma + 1)^2} = \sqrt{(2t)^2 + 2t} < 2t+1$. A similar argument to that above shows that this ball contains all blue points, but none of the red.

We show that under the NN rule, at least $n-2$ points must be found in the condensed set. First take the blue points on the line $x=1$, and at least one of these must be in the condensed set: If only $b$ were in the condensed set, then the other blue points would be closer to every red point in the condensed set. Now take a blue point $(1,2i + 1)$ in the condensed set, and it must be that the red points $(0,2i),(0,2i+2))$ are both in the condensed set, since the blue point is their nearest neighbor. It similarly follows that blue points $(1,2i-1),(1,2i+3)$ are both in the condensed, and in fact that all red points on the line $x=0$ and all blue points on the line $x=1$ are in the condensed set.
\end{proof}

\subsection{Learning bounds}
\label{learningbounds}
Here we establish uniform generalization bounds for WNN rules. By definition, a general WNN classifier is uniquely determined by a subset of labeled samples $\dsc$ and a weight function $\w$.
In addition, since $\w(x) = 1$ for all $x\notin \dsc$, $\w$ is uniquely determined by its values on $\dsc$.
The following representation lemma establishes a $2|\dsc|$-sample compression scheme for the class of WNN rules, which will allow us to derive compression-based error bounds \citep{warmuth86,floyd1995sample,graepel2005pac}.

\newcommand{\cons}{\mathcal C}
\begin{lemma}
\label{lem:representer}
There exist an 
encoding function $\cons$ and a reconstruction function $\recon$ that satisfy the following:
For any finite labeled set $\ds$, any subset $\dsc\subset\ds$, and any weight assignment $\w$ for $\dsc$, if the WNN classifier corresponding to the pair $(\dsc,\w)$ is consistent on $\ds$, then $\cons(\ds,\dsc,\w)$ returns two ordered labeled subsets $\dsc',\wset'\subset\ds$, each of size $|\dsc|$, such that the WNN classifier $\cl_{\dsc',\wset'}=\recon(\dsc',\wset')$ is consistent on $\ds$.
\end{lemma}

\begin{proof}
We first describe the encoding function $\cons$ that given $\ds$ and the pair $(\dsc,\w)$, selects $\dsc',\wset'$. Then we will describe the reconstruction function $\recon$ and conclude that $
\recon(\dsc',\wset')$ is indeed consistent on $\ds$.

Given $\ds$ and $(\dsc,\w)$, the encoding function $\cons$ 
returns two ordered lists $\dsc',\wset'\subset \ds$ such that $\dsc'$ has the same sample points in $\dsc$ but in a specific order, and $\wset'$ consists of sample points from $\ds$ from which appropriate weights for $\dsc'$ can be deduced via $\recon$.
The first sample in the list $\dsc'$ corresponds to $x\in\dsc$ having the maximal weight, and 
the first sample in $\wset'$ is also taken as $x$. Subsequent points are added to $\dsc'$ and $\wset'$ by the following procedure: Starting with initial weights as determined by $\w$, multiplicatively increase the weights of all points of $\dsc$ not yet placed in $\dsc'$, until one of the following occurs:
\begin{itemize}[noitemsep,leftmargin=*]
\item[{(i)}] A point $x \in \dsc\setminus\dsc'$ has weight equal to that of a point already in $\dsc'$, say $x_1$. Then $x$ is added to $\dsc'$ and $x_1$ is added to $\wset'$. 
\item[{(ii)}] There is some point $x \in S \setminus \Tilde{S}$ that became equidistant (under weighted distance) from its closest point $x_1 \in \Tilde{S}$ of the same label, and its closest point $x_2 \in \Tilde{S}$ with opposite label. It must be that $x_1$ is already in $\Tilde{S'}$, or else the weights of $x_1,x_2$ would increase in unison.
Then $x_2$ is added to $\dsc'$ and $x$ is added to $\wset'$.
\end{itemize}
The weight increase procedure is carried on until all samples of $\dsc$ have been placed in $\dsc'$.

As for the reconstruction, given two lists of sample point-label pairs $\dsc'=( (x_1,y_1),\dots,(x_m,y_m))$ and $\wset'=((x'_1,y'_1), \dots, (x'_m,y'_m))$ that have been computed by $\cons$ as described above, the reconstruction function $\recon$ construct a WNN classifier corresponding to the pair $(\dsc',\w')$, where the weight assignment $\w'$ is computed as follows. The weight of the first point in $\dsc'$ is set to $\w'(x_1)=1$. The weights of subsequent points in $\dsc'$ are set depending on their corresponding points in $\wset'$: For $k>1$,
\begin{itemize}[noitemsep,leftmargin=*]
    \item If $x_k$ appears in $(x_1,\dots,x_{k-1})$, say as $x_j$, then we are in case (i), and thus put $\w'(x_k) = \w'(x_j)$.
    
    \item If $x_k$ does not appear in $(x_1,\dots,x_{k-1})$, then we are in case (ii), and $(x'_k,y'_k)$ corresponds to a witness for $(x_k,y_k)$ and some other point in $(x_1,\dots,x_{k-1})$, say $(x_j,y_j)$. The identity of $x_j$ can be inferred by finding the point in $(x_1,\dots,x_{k-1})$ with label opposite to $y_k$ and having the minimal weighted distance to $x'_k$,
    \begin{align*}
        x_j = \argminn_{x_i \in \{x_1,\dots,x_{k-1}\}: y_i\neq y_k} \left\{ \frac{\dist(x'_k,x_i)}{\w'(x_i)}\right\},
    \end{align*}
    breaking ties towards $x_i$ with smallest index $i$.
    Then $\w'(x_k)$ is set to satisfy the equation 
    \begin{align*}
        \frac{\dist(x'_k,x_k)}{\w'(x_k)} = \frac{\dist(x'_k,x_j)}{\w'(x_j)} .
    \end{align*}    
\end{itemize}

Since multiplying the weights of several points in unison do not change their pairwise weighted-distance boundaries, it is clear that during the whole construction process done in $\cons$, the classifier remains consistent on $\ds$, provided ties in weighted distances are decided in favor of points in $\dsc'$ of smaller index. Hence $\recon(\dsc',\wset')$ is consistent on $\ds$.
\end{proof}

Lemma \ref{lem:representer} can be used to derive generalization bounds, as we show in Theorem \ref{thm:generalization bound}. But note first that the reconstruction function $\recon$ of Lemma \ref{lem:representer} heavily relies on that $\dsc'$ and $\wset'$ are ordered.
Below in Section \ref{sec:heuristic} we will consider a subclass of WNNs whose $\recon$ assumes no such matching is given, and this matching needs to be deduced. In this case a tighter generalization bound holds, which we will leverage to establish Bayes consistency for the aforementioned subclass in Section \ref{sec:Bayes consistency}.
Formally, a reconstruction function $\recon$ is said to be \emph{permutation invariant} if for any two arbitrary permutations $\sigma_1,\sigma_2$ of the samples in $\dsc'$ and $\wset'$ respectively, 
\begin{align}
\label{eq:permutation equality}
    \recon(\sigma_1(\dsc'),\sigma_2(\wset')) = \recon(\dsc',\wset'). 
\end{align}
In other words, a permutation invariant $\recon$ is able to deduce the matching between the samples in $\dsc'$ and their weights from the unordered elements of $\dsc'$ and $\wset'$ alone. 

We can now present the generalization bounds. These essentially correspond to sample compression-based error bounds for the compression scheme established in Lemma \ref{lem:representer}. Similar finite-sample bounds (up to constants) for standard NN rules have been established in \citet{DBLP:journals/tit/GottliebKK14+colt,DBLP:conf/nips/KontorovichSW17,HKSW}; in particular, leveraging the fact that such bounds are dimension free, \citet{HKSW} derived similar bounds to establish universal Bayes consistency of a nearest neighbor rule (OptiNet) in any essentially-separable metric space.
 
\begin{theorem}
\label{thm:generalization bound}
For any probability distribution of $x$, any labeling function $l:\X\to\{-1,1\}$, and any $n\in\mathbb N$ and $0<\delta<1$, it holds that with probability $1-\delta$ over the i.i.d.\ labeled sample $S = \left\{(x_1,l(x_1)),\dots,(x_n,l(x_n))\right\}$, for any $\dsc\subset \ds$ and weight assignment $\w$ for $\dsc$, if the WNN classifier $\cl_{(\dsc,\w)}$ corresponding to the pair $(\dsc,\w)$ correctly classifies all points in $S$, then
\begin{align}
\label{eq:err_bound}
    \err(\tcl_{(\dsc,\w)}) \leq
    \frac{2}{n-|\dsc|}\left(|\dsc|\log 2n + \log \frac{n}{\delta}\right).
\end{align}
If in addition the reconstruction function $\recon$ is permutation invariant, then
\begin{align}
\label{eq:err_bound_perm}
    \err(\tcl_{(\dsc,\w)}) \leq
    \frac{2}{n-|\dsc|}\left(|\dsc|\log{\left(\frac{2en}{|\dsc|}\right)} + \log \frac{n}{\delta}\right).
\end{align}
\label{thm:err_comp}
\end{theorem}
\begin{proof}[Proof of Theorem \ref{thm:generalization bound}.]
Consider first the case of non-invariant $\recon$. Set $m=|\dsc|$ and let $\I_{n,m}$ denote the set of all (ordered) sequences  of length $m$ of indices from $\{1,2,\dots,n\}$.
For $\bm i = (i_1,\dots,i_m)\in\I_{n,m}$ denote by $\ds(\bm i)$ the subset of $\ds$ with indices in $\bm i$.
For a weight assignment $\w'(\bm i)$ for $\ds(\bm i)$, write $\emperr(\cl_{(\ds(\bm i), \w'(\bm i))})$ for the empirical error of $\cl_{(\ds(\bm i), \w'(\bm i))}$ over the $n-m$ samples in $\ds\setminus \ds(\bm i)$, that is
\begin{align*}
    \frac{1}{n-m} \sum_{(x_j,y_j)\in \ds\setminus \ds(\bm i)}
    \chr{\cl_{(\ds(\bm i), \w'(\bm i))}(x_j) \neq y_j},
\end{align*}
and note that the samples in $\ds\setminus \ds(\bm i)$ are i.i.d.\ and independent of 
$\ds(\bm i)$.
For $\eps>0$, we bound
\begin{align}
\nonumber
 &\PROB\Big\{\emperr(\cl_{(\dsc,\w)})=0
  \,\,\wedge\,\, 
 |\dsc|=m
 \,\,\wedge\,\, 
 \err(\cl_{(\dsc,\w)})>
 \eps
  \Big\}
\\[3pt]
 \nonumber
 \leq  \,\, &
 \PROB\Big\{\exists 
 \bm i\in\I_{n,m}, \exists \w'(\bm i):
 \emperr(\cl_{(\ds(\bm i),\w'(\bm i))})=0
\,\,\wedge\,\,
 \err(\cl_{(\ds(\bm i),\w'(\bm i))})>
 \eps
  \Big\}
 \\[3pt]
 \label{eq:exp_prob}
 \leq  \,\, &
 \sum_{\bm i\in\I_{n,m}}
 \EXP\Big\{
 \PROB\big\{\exists \w'(\bm i):
  \emperr(\cl_{(\ds(\bm i),\w'(\bm i))})=0
\,\,\wedge\,\,
 \err(\cl_{(\ds(\bm i),\w'(\bm i))})> \eps  \mid \ds(\bm i)\big\} \Big\}.
\end{align}

Fix $\bm i\in\I_{n,m}$ and $\ds(\bm i)$, and consider the class of WNN classifiers given by
\begin{align*}
    \mathcal H_{\ds(\bm i)} = \{\cl_{(\ds(\bm i),\w'(\bm i))}: \text{$\w'(\bm i)$ weight assign.\ for $\ds(\bm i)$}\}.
\end{align*}
The growth function $\Pi_{\mathcal H_{\ds(\bm i)}}(n)$
of $\mathcal H_{\ds(\bm i)}$ counts the maximum number of different possible labelings of $n$ points from $\X$:
\begin{align*}
    &\Pi_{\mathcal H_{\ds(\bm i)}}(n)=
    \max_{\{z_1,\dots,z_{n}\}\subset \X}
    \Big|\big\{ (h(z_1),\dots,h(z_{n})) : h\in \mathcal H_{\ds(\bm i)} \big\}\Big|.
\end{align*}
Then by a standard argument \citep{mohri2018foundations},
\begin{align*}
& \PROB\big\{\exists \w'(\bm i):
  \emperr(\cl_{(\ds(\bm i),\w'(\bm i))})=0
\,\,\wedge\,\,
 \err(\cl_{(\ds(\bm i),\w'(\bm i))})> \eps  \mid \ds(\bm i)\big\}
 \\[3pt] &\leq 
 2\Pi_{\mathcal H_{\ds(\bm i)}}(2(n-m)) \cdot e^{-(n-m)\eps/2}.
\end{align*}
To bound $\Pi_{\mathcal H_{\ds(\bm i)}}(2(n-m))$,
note that by Lemma \ref{lem:representer}, for any weight assignment $\w'(\bm i)$ for $\ds(\bm i)$ and any set $\ds'$ of $2(n-m)$ points from $\X$, there exists a subset $\wset'\subset \ds' \cup \dsc(\bm i)$ of size $m$ such that the WNN classifier $\cl_{\ds(\bm i),\wset'}$ gives the same labeling as $\cl_{(\ds(\bm i),\w'(\bm i))}$ on the $2(n-m)$ points in $\ds'$.
Since the number of different classifiers in 
\begin{align*}
\Big\{\cl_{\ds(\bm i),\wset'}: \text{$\wset'$ is a list of $m$ points}
\text{from a sample of size $2n-m$}
\Big\}    
\end{align*}
is at most $|\I_{2n-m,m}| \leq |\I_{2n,m}|$, it follows that 
\begin{align*}
    \Pi_{\mathcal H_{\ds(\bm i)}}(2(n-m)) \leq |\I_{2n,m}|.
\end{align*}
Hence, Eq.\ \eqref{eq:exp_prob} is bounded from above by
\begin{align}
\sum_{\bm i\in\I_{n,m}} |\I_{2n,m}| e^{-(n-m)\eps/2}
\label{eq:exp_I}
\leq
|\I_{n,m}|
|\I_{2n,m}| 
e^{-(n-m)\eps/2}.
\end{align}
Put
\begin{align}
\label{eps_set}
\eps &= \frac{2}{n-m}\left(\log{(|\I_{n,m}|\cdot|\I_{2n,m}|)} + \log \frac{n}{\delta}\right)    
\leq 
\frac{2}{n-m}\left(m\log{2n} + \log \frac{n}{\delta}\right)    
,
\end{align}
where we used the bound $|\I_{n,m}| \leq n^m$. Then the right hand side of \eqref{eq:exp_I} is $\delta/n$.
Summing over the $n$ possible values of $m$ completes the proof for the case of non-invariant $\recon$. 

As for the case of permutation invariant $\recon$, the only difference from the proof above is the definition of $\I_{n,m}$. For the permutation invariant case we take $\I_{n,m}$ to be the set of all (unordered) subsets of $\{1,2,\dots,n\}$ of size $m$.  Then $|\I_{n,m}| \leq {n \choose m} \leq (\frac{en}{m})^m$. Putting this into Eq. \eqref{eps_set}, we have
\begin{align*}
\eps & \leq 
\frac{2}{n-m}\left(m\log{\frac{2en}{m}} + \log \frac{n}{\delta}\right),
\end{align*}
in accordance with \eqref{eq:err_bound_perm}.
\end{proof}


\section{Greedy heuristic for WNN, and its properties}\label{sec:heuristic}

In this section, we suggest a heuristic to produce a weighted condensed set. The heuristic is motivated by the ball cover rule introduced above. After presenting the heuristic, we establish that it is Bayes consistent under mild assumptions.

\subsection{Heuristic}
Our heuristic for WNN condensing is based on a greedy approach for the BC rule, meaning that at each step, we identify a ball covering the maximum number of uncovered points of the same label (and no points of the opposite label), and add it to our ball set. 
Similarly, in our WNN heuristic (see Algorithm \ref{alg:weighted-heuristic}), we iteratively add to the condensed set a point and weight which correspond to the center and radius of the ball covering the largest number of as of yet not covered points. (The equivalence of the radius in the BC rule to weight in the WNN rule was already established above in Section \ref{sec:condensing-bounds}.) 
As in \cite{FM-21}, the notation $d_{\enemy}(x)$ denotes the distance from $x$ to its closest oppositely labeled point in $S$ (its `nearest enemy').

\begin{algorithm}[t]
\caption{Greedy weighted heuristic}\label{alg:weighted-heuristic}
\begin{algorithmic}
\State Input: Point set $S$
\State Initialize solution set $T \gets \emptyset$, $S' \gets S$, weight function $w: S \rightarrow \{1\}$
\While{$S' \ne \emptyset$}
    \State $x \gets \argmax_{x \in S} |B(x,d_{\enemy}(x)) \cap S'| $
    \State $S' \leftarrow S' \setminus B(x,d_{\enemy}(x))$
    \State $T \gets T \cup \{x\}$
    \State $\w(x) \leftarrow d_{\enemy}(x)$
\EndWhile
\State return $T, \w$.
\end{algorithmic}
\end{algorithm}

\subsection{Bayes Consistency}
\label{sec:Bayes consistency}


In the following, we consider a family of WNN classifiers that use a specific (data-dependent) weight function $\w_{\enemy}$,
that assigns to each data point $(\tX,\tY)\in\dsc\subset\ds$ weight equal to the minimal distance from $\tX$ to the points in $\ds$ that have the opposite label from $\tY$, and for points $x \notin \dsc$ assigns 
$\w_{\enemy}(x)=1$; 
that is,
defining $\dsp$ and  $\dsn=\ds\setminus \dsp$ as the split of $\ds$ into positively and negatively labeled points respectively,
$\w_{\enemy}(\tX) = \dist_{\enemy}(\tX),$
where
\begin{align*}
    \dist_{\enemy}(\tX) =
    \begin{cases}
        \dist(\tX,\dsn),& \qquad \tX\in \dscp,
     \\
     \dist(\tX,\dsp),& \qquad \tX\in \dscn.
    \end{cases}
\end{align*}

Note that for this subclass of WNN classifiers there is a simpler compression scheme than that of Lemma \ref{lem:representer}; in particular, the reconstruction function $\recon$ can be made permutation invariant in the sense of \eqref{eq:permutation equality}.
Indeed, the weight assignment $\w_{\enemy}$ for $\dsc$ can be encoded into a subset $\wset\subset\ds\setminus\dsc$ consisting of the nearest enemies of the samples in $\dsc$. Then the weight for $(\tX,\tY)\in\dsc$ can be uniquely determined by splitting $\wset$ into its positively and negatively labeled samples $\wset^+$ and $\wset^-$ and putting 
$
\w_{\enemy}(\tX) = \min_{(x,y)\in \wset^{-\tY}} \dist(\tX,x).
$
With this rule, for any two permutations $\sigma_1,\sigma_2$, $\recon(\sigma_1(\dsc),\sigma_2(\wset)) = \recon(\dsc,\wset)$.

We first consider the (computationally intractable) learning rule that finds the subset $\dsco\subset\ds$ of minimal cardinality such that the classifier $\cl_{(\dsco,\w_{\enemy})}$ is consistent on $\ds$,
\begin{align}
\label{eq:bc_most_comp}
    \dsco = \argminn_{\dsc \subset S}\{|\dsc| : \tcl_{(\dsc,\w_{\enemy})}(x)=y 
    ,\forall (x,y)\in \ds 
    \}.
\end{align}
The following theorem establishes the Bayes consistency of $\tcl_{(\dsco,\w_{\enemy})}$.

\begin{theorem}
    Let $(\X,\dist)$ be a separable metric space and assume $x$ has an atomless distribution and that the labeling function $l$ is countably piece-wise continuous.
    Then, almost surely,
    \begin{align}
    \label{eq:bc_thm}
       \err(h_{(\dsco,\w_{\enemy})}) \xrightarrow[n\to\infty]{} 0. 
    \end{align}
    \label{thm:bc_most_comp}
\end{theorem}

The proof below essentially establishes that $|\dsco|$ is almost surely sub-linear in the sample size $n$. An application of the error bound \eqref{eq:err_bound_perm} of Theorem \ref{thm:generalization bound} (corresponding to a permutation-invariant rule) then establishes \eqref{eq:bc_thm}. Note that a sub-linear $|\dsco|$ in conjunction with the error bound \eqref{eq:err_bound} (corresponding to a non-permutation invariant rule) do not suffice to establish Bayes consistency with our proof technique: Without further assumptions on the tail of the distribution of $x$, the rate at which $|\dsco|/n \xrightarrow[n\to\infty]{} 0$ can be arbitrarily slow.

\begin{proof}[Proof of Theorem \ref{thm:bc_most_comp}.]
Denote by $\mu$ the probability distribution of $x$ and abbreviate $\cl_{\dsco}=\cl_{(\dsco,\w_{\enemy})}$.
For $r>0$ let 
\begin{align*}
    \unifset_r = \left\{ x\in\supp(\mu): 
    \frac{1}{\mu(\ball_r(x))}\int_{\ball_r(x)} l(x')\mu(dx')
    = l(x)
    \right\};
\end{align*}
that is, $\unifset_r$ is the set of all points in the support of $\mu$ where $l$ is essentially constant in the ball of radius $r$ around $x$.
The assumptions that $x$ is atomless and that $l$ is piece-wise continuous imply that $\mu(U_r)$ is monotonic decreasing and continuous in $r$ and satisfies
\begin{align}
\label{eq:piecewise_cond}
\lim_{r\to0}\mu(\unifset_r)=1.    
\end{align}
Given $\eps>0$, let $\rado=\rado(\eps)>0$ be such that
\begin{align*}
    \mu(\unifset_{\rado}) \geq 1- \alpha,
\end{align*}
where $\alpha = \alpha(\eps)\in(0,1/8)$ satisfies
\begin{align}
\label{eq:alpha_def}
\frac{8\alpha\log\left(\frac{e}{2\alpha}\right)}{1-4\alpha} \leq \frac{\eps}{2}.
\end{align}
Since the left hand side of \eqref{eq:alpha_def} goes to 0 monotonically (and continuously) as $\alpha\to 0$, such an $\alpha$ always exists (this choice of $\alpha$ will be made clear below).

Given a sample $\ds$ of size $n$, denote by $\bX_n=(x_1,\dots,x_n)$ the instances in $\ds$ and by $\bY_n=(y_1,\dots,y_n)$ their corresponding labels.
Let $\net_{\rado}\subseteq \bX_n\cap \unifset_{\rado}$ be an $\rado$-net of $\bX_n\cap \unifset_{\rado}$ 
(an $r$-net of a set is defined in Appendix \ref{sec:search})
and let $\bY_{\rado}$ be the corresponding labels in $\bY_n$, stacked into the labeled set $\dsc_{(1)}=(\bX_{\rado},\bY_{\rado})$. 
Let $\unifset_{\rado}^c= \X\setminus\unifset_{\rado}$ denote the set complement of $\unifset_{\rado}$  and let $\dsc_{(2)}=S\cap (\unifset_{\rado}^c \times \Y)$.
 Define the labeled dataset 
\begin{align*}
    \dsc_{\rado} = 
    \dsc_{(1)} \cup \dsc_{(2)} \subset \ds.
\end{align*}
Then $\cl_{\dsc_{\rado}}$ is consistent on $\ds$. Indeed, any $(x_i,y_i)\in \ds\cap (\unifset_{\rado}^c\times \Y)$ is included in $\dsc_{\rado}$ and is thus classified correctly by $\cl_{\dsc_{\rado}}$.
For any $(x_i,y_i) \in \ds\cap (\unifset_{\rado}\times \Y)$, since $\bX_{\rado}$ is an $\rado$-net  of $\bX_n\cap \unifset_{\rado}$, there is $\tX\in \bX_{\rado}$ with 
\begin{align*}
    \dist(x_i,\tX) < \rado.
\end{align*}
Since $\tX\in\unifset_{\rado}$ and $\dist(x_i,\tX) < \rado$, we have $y_i=l(x_i)=l(\tX)=\tY$ with probability 1. In addition, 
any point $(x_j,y_j)\in\ds$ with an opposite label to $\tY$ satisfies 
$\dist(x_j,\tX) \geq \rado$,
and so $\w_{\enemy}(\tX)\geq \rado$.
Thus,
\begin{align*}
    \distw(x_i,\tX) = \frac{\dist(x_i,\tX)}{\w_{\enemy}(\tX)} < 1.
\end{align*}
Additionally,
any $\tX'\in\dsc_{\rado}$ with a different label from $y_i$ has weight
\begin{align*}
\w_{\enemy}(\tX')\leq \dist(x_i,\tX').
\end{align*}
Thus,
\begin{align*}
    \distw(x_i,\tX') = \frac{\dist(x_i,\tX')}{\w_{\enemy}(\tX)} \geq 1.
\end{align*}
So the WNN classifier $\cl_{\dsc_{\rado}}$ classifies the point $(x_i,y_i)$ correctly in this case as well, and so $\cl_{\dsc_{\rado}}$ is consistent on $\ds$.
It follows that the subset $\dsco$ in \eqref{eq:bc_most_comp} computed by the learning rule satisfies
\begin{align}
\label{eq:opt_smaller_net}
    |\dsco| \leq |\dsc_{\rado}|.
\end{align}

We next bound $|\dsc_{\rado}| = |\dsc_{(1)}| +  |\dsc_{(2)}|$ with high probability.
Since by construction $\mu(U_{\rado}^c)\leq \alpha$, Hoeffding's inequality implies that
\begin{align*}
\PROB\left\{|\dsc_{(2)}|  > 2n \alpha\right\} =
    \PROB\left\{|\bX_n \cap \unifset_{\rado}^c| > 2n \alpha\right\} \leq e^{-2n\alpha^2}.
\end{align*}
As for $|\ds_{(1)}|$,
by \citet[Lemma 3.7]{HKSW}, 
there is  $t_{\rado}:\mathbb N\to \mathbb R^+$ in $o(1)$ such that
\begin{align*}
    \PROB\left\{|\net_{\rado}|\geq n t_{\rado}(n)\right\} \leq  1/n^2.
\end{align*}
Since $t_{\rado}\in o(1)$, we may take $n$ sufficiently large so that $t_{\rado}(n)\leq2\alpha$.
So for all sufficiently large $n$,
\begin{align*}
    \PROB\{|\dsc_{\rado}|> 4\alpha n \} \leq \frac{1}{n^2} + e^{-2n\alpha^2 }.
\end{align*}

We bound
\begin{align}
\nonumber
\PROB\left\{\err(\cl_{\dsco})>\eps\right\}  
& \leq 
\PROB\left\{\err(\cl_{\dsco})>\eps \,\,\wedge\,\, |\dsc_{\rado}|\leq 4\alpha n\right\}
+ \PROB\left\{|\dsc_{\rado}|> 4\alpha n\right\}
\\
\label{eq:sum_bound}
&\leq 
\PROB\left\{\err(\cl_{\dsco})>\eps \,\,\wedge\,\, |\dsc_{\rado}|\leq 4\alpha n\right\}
+
\frac{1}{n^2} + e^{-2n\alpha^2}.
\end{align}
To complete the proof we show below that for all sufficiently large $n$,
\begin{align}
\label{eq:comp_prob}
    \PROB\left\{\err(\cl_{\dsco})>\eps \,\,\wedge\,\, |\dsc_{\rado}|\leq 4\alpha n\right\}
    \leq \frac{1}{n^2}.
\end{align}
Since the right hand side of \eqref{eq:sum_bound} is summable over $n$, the Borel-Cantelli Lemma implies that almost surely,
\begin{align*}
\err(h_{\dsc^*}) \xrightarrow[n\to\infty]{} 0, 
\end{align*}
concluding the proof of the Theorem.

To show \eqref{eq:comp_prob},
put $\delta=\delta_n=1/n^2$ 
in \eqref{eq:err_bound_perm} and observe that the 
right hand side 
of \eqref{eq:err_bound_perm} 
is monotonic increasing with $|\dsc|$. 
Thus, using \eqref{eq:opt_smaller_net},
we have that under the event $\{|\dsc_{\rado}|\leq 4\alpha n\}$,
\begin{align*}
\frac{2}{n-|\dsco|}\left(|\dsco|\log{\left(\frac{2en}{|\dsco|}\right)} + 3\log n\right)
&\leq
\frac{2}{n-|\dsc_{\rado}|}\left(|\dsc_{\rado}|\log{\left(\frac{2en}{|\dsc_{\rado}|}\right)} + 3\log n\right)
\\ &\leq
\frac{2}{n-4\alpha n}\left(4\alpha n\log{\left(\frac{en}{2\alpha n}\right)} + 3\log n\right)
\\
&=
\frac{8\alpha\log{\left(\frac{e}{2\alpha}\right)}}{1-4\alpha}+ \frac{3\log n}{(1-4\alpha)n}
\\ &\leq \frac{\eps}{2} + \frac{\eps}{2} = \eps,
\end{align*}
where in the last inequality we used the choice of $\alpha$ in \eqref{eq:alpha_def} and took $n$ sufficiently large so that $\frac{3\log n}{(1-4\alpha)n}\leq\frac{\eps}{2}$.
Applying Lemma \ref{thm:err_comp}, it follows that 
\begin{align*}
     \PROB\left\{\err(\cl_{\dsco})>\eps \,\,\wedge\,\, |\dsc_{\rado}|\leq 4\alpha n\right\}
     &\leq 
      \PROB\left\{\err(\cl_{\dsco})>\frac{2}{n-|\dsco|}\left(|\dsco|\log{\left(\frac{2en}{|\dsco|}\right)} + 3\log n\right)\right\}
      \leq \frac{1}{n^2},
\end{align*}
establishing \eqref{eq:comp_prob}.
\end{proof}

As for our greedy heuristic in Algorithm \ref{alg:weighted-heuristic}, note that the intractable optimization problem \eqref{eq:bc_most_comp} can be cast as a set cover problem.  Algorithm \ref{alg:weighted-heuristic} then corresponds to the standard greedy approximation for set cover \citep{chvatal1979greedy}. This approximation is guaranteed to compute $\dsc$ of size at most $O(|\dsco|\log |\dsco|)$. Hence, if $|\dsco|\log |\dsco|$ is guaranteed to be almost surely sub-linear, an adaptation of our proof of Theorem \ref{thm:bc_most_comp} implies that the greedy heuristic is Bayes consistent. This is made formal in the following Corollary. 

\begin{corollary}
\label{cor:bc_heuristic}
Under the conditions of Theorem \ref{thm:bc_most_comp} and an additional appropriate tail condition on the probability distribution of $x$, the greedy weighted heuristic of Algorithm \ref{alg:weighted-heuristic} is Bayes consistent.
\end{corollary}


\begin{proof}[Proof of Corollary \ref{cor:bc_heuristic}]
In the proof of Theorem \ref{thm:bc_most_comp} we fetched a function $t_{\rado}:\mathbb N\to \mathbb R^+$ in $o(1)$ to establish that the size of a $\rado$-net (with $\rado>0$) is sub-linear in $n$.
Inspecting the proof of Theorem \ref{thm:bc_most_comp}, to guarantee that $|\dsco|\log |\dsco|$ is almost surely sub-linear in $n$, it suffices to additionally assume that $t_{\rado}\in o(1/\log n)$.
\end{proof}

As two examples of the applicability of Corollary \ref{cor:bc_heuristic}, if random variable $x$ is bounded then $|\dsco|= O(1)$, and if $x$ has a normal distribution then $|\dsco|= O(\log n)$. Hence in these examples Algorithm \ref{alg:weighted-heuristic} is Bayes consistent.

\section{Search algorithm}\label{sec:search}

\algnewcommand\algorithmicforeach{\textbf{for each}}
\algdef{S}[FOR]{ForEach}[1]{\algorithmicforeach\ #1\ \algorithmicdo}

In previous sections, we motivated the WNN rule, and gave a heuristic for producing a condensed set and weights. In this section, we consider the algorithmic problem of, given a weighted condensed set $\Tilde{S}$ and weight function $w$, labeling a new point $q$ (the \emph{query}) via the WNN rule on $\Tilde{S},w$. This is the same as finding the weighted distance nearest neighbor of $q$ in $\Tilde{S}$

We note that if the number of distinct weights in $w$ is at most $k$, then we can reduce the above problem to solving $k$ separate (unweighted) nearest neighbor problems: Group the points of $\Tilde{S}$ by into at most $k$ groups, each containing only points of the same weight. For each group, preprocess an unweighted nearest neighbor search structure. Given the query $q$, find the nearest unweighted neighbor of $q$ within each group. One of these $k$ points must be the weighted distance nearest neighbor.

We will give an approximate search algorithm with run-time independent of the number of distinct weights. This algorithm is a modification of the navigating net algorithm for doubling spaces \cite{KL04}, and we first describe this structure.

\paragraph{Navigating nets.}
For a metric $(\X,d)$, let $\lambda$ be the smallest value such that every ball in $\X$ of radius $r$ (for any $r$) can be covered by $\lambda$ balls with radius $\frac{r}{2}$. $\lambda$ is the \emph{doubling constant} of $\X$, and $\log_2 \lambda$ is the \emph{doubling dimension} of $\X$ ($\ddim := \ddim(\X)$). 
Doubling spaces also obey a \emph{packing property}: The number of points at minimum inter-point distance at least $b$ within a ball of radius $r$ is at most $(r/b)^{O(\ddim)}$.

An $r$-\emph{net} of $\X$ is a subset $S \subset \X$ satisfying the following:
\begin{itemize}[noitemsep,leftmargin=*]
    \item Packing: The minimum inter-point distance in $S$ is greater than $r$.
    \item Covering: For every $x \in \X$, $d(x,S) \le r$.
\end{itemize}

A \emph{navigating net} is a nested series of $r$-nets, each set a net of the subsequent one with twice its radius. Assume without loss of generality that the smallest inter-point distance in $\X$ is exactly 1, and let $t$ be the smallest integer satisfying $2^t \ge \diam(\X)$. Then the navigating net is composed of a series of $r$-nets $\{S_{2^t},S_{2^{t-1}},\ldots,S_1\}$ where for all nets $S_{2^{i+1}} \subset S_{2^i}$. This series of $r$-nets is called a \emph{hierarchy}. 

A tree is then defined over the hierarchy: To every point in each level of the hierarchy we associate a unique vertex, so that each point in $S_{2^i}$ corresponds to some $i$-level vertex. To an $i$-level vertex associated with some point $p \in S_{2^i}$ ($i<t$) we assign a \emph{parent}, a single $(i+1)$-level node associate with (an arbitrary) point in $S_{2^{i+1}}$ within distance $2^{i+1}$ of $p$. The navigating net is composed of the hierarchy and its associated tree, and can be computed in time $2^{O(\ddim)} t$.

\paragraph{Our solution.}
Given a navigating net, we will need to make a simple modification to it: For every node $v \in T$, we record heaviest point among all the points associated with the nodes in the sub-tree rooted at $v$.

A $(1+\epsilon)$-approximate weighted nearest neighbor query (for $\epsilon<1$) proceeds as follows: Starting from the root at level $t$, the query descends down the levels of the tree. At each level $i$, we compute a list of all nodes whose associated points are within (unweighted) distance $\frac{2 \cdot 2^i}{\epsilon}$ of query point $q$ (and by the packing property of doubling spaces, there are $\epsilon^{-O(\ddim)}$ such nodes). This can be done easily, as by the triangle inequality, an $i$-level nodes within this distance of $q$ has its $(i+1)$-level parent within distance

\begin{equation*}
\frac{2 \cdot 2^i}{\epsilon} + 2^{i+1} 
= \frac{2^{i+1}}{\epsilon} + 2^{i+1}
< \frac{2 \cdot 2^{i+1}}{\epsilon},
\end{equation*}
and so this parent was found in the previously computed list for level $i+1$. So the node list for level $i$ can be computed by inspecting all children of the nodes in the list for level $i+1$, and retaining only those nodes whose associated points are sufficiently close to query $q$. For each node in every list, the algorithm queries the weighted distance of the associated point to $q$. The point with the smallest weighted distance is returned. See Algorithm \ref{alg:nn2-query}, and its run time in Lemma \ref{lem:nn2}.

\begin{algorithm}
\caption{Weighted distance nearest neighbor}\label{alg:nn2-query}
\begin{algorithmic}
\State Input: Modified navigating net tree $T$, query $q$
\State Initialize nearest $\gets \emptyset$, $L \gets \{ \text{root of } T\}$
\ForEach{$i=t,\ldots,1$}
    \State $L' \gets \emptyset$
    \ForEach{$v \in L$}
        \ForEach{child $w$ of $v$}
            \If{$d(q,w) \le \frac{2 \cdot 2^i}{\epsilon}$}
                \State $L' \gets L' \cup \{w\}$
                \If{$\tilde{d}(q,w) < \tilde{d}(q,\text{nearest})$}
                    \State nearest $\gets w$
                \EndIf
            \EndIf
        \EndFor
    \EndFor
    \State $L \gets L'$
\EndFor
\State return nearest
\end{algorithmic}
\end{algorithm}

\begin{lemma}\label{lem:nn2}
Algorithm \ref{alg:nn2-query} returns a $(1+\epsilon)$-approximate weighted distance nearest neighbor of $q$ in time $\epsilon^{-O(\ddim)}t$.
\end{lemma}

\begin{proof}
Let $p^*$ be the weighted nearest neighbor of $q$, and let $i$ be the integer satisfying $d(q,p^*) < \frac{2^i}{\epsilon} \le 2d(q,p^*)$. Let $v$ be the $i$-level tree ancestor of the leaf node associated with $p^*$, and we have by the triangle inequality that 
$d(q,v) 
\le d(q,p^*) + d(p^*,v) 
< \frac{2^i}{\epsilon} + 2 \cdot 2^i 
< \frac{2 \cdot 2^i}{\epsilon}$. 
It follows that $v$ is found in the $i$-level list. If $p^*$ is the heaviest ancestor of $v$, then it will be encountered by the algorithm and returned. If a different point $p$ is stored in $v$, then necessarily $w(p) \ge w(p^*)$. And further,
$d(q,p) 
\le d(q,v) + d(v,p)
\le d(q,p^*) + d(v,p^*) + d(v,p)
\le d(q,p^*) + 4 \cdot 2^i
\le (1+8\epsilon) d(q,p^*)
$,
so that $p$ is a $(1+8\epsilon)$-approximate nearest neighbor.
\end{proof}

\section{Experimental results}
\label{sec:results}

In this section we present promising experimental results for condensing under WNN, using the heuristic of Section \ref{sec:heuristic}. We ran two separate trials: one on small data sets, and the other on a large data set. Each of these trials contribute to understanding and quantifying the quality of our heuristic.

\subsection{Small dataset trial}
The first trial was on relatively small datasets. For these sets, we can use an integer program (IP) to compute the optimal solution under NN condensing, which is instructive in understanding the quality of the tested heuristics.

\paragraph{Integer programming for NN condensing.} We formalize an integer program for NN condensing. As this problem is NP-hard, we do not expect the algorithm to have a reasonable run time for large sets, but for smaller sets it successfully returns an optimal solution (after large run time).

To model the NN condensing problem as an integer program, we introduced constraints corresponding to the inclusion of a point in the condensed set. This allowed us to identify the minimal non-empty subset of examples that can recover all labels of the sample via the nearest neighbor classifier.

For each sample point $x$, we introduce a 0–1 variable $v(x)$, corresponding to whether $x$ will appear in the condensed set. For each ordered pair of points $x$ and $x'$ with opposite labels, we introduce the constraint
$
v(x') \leq \sum_{x'' \in C(x,x')} v(x'')
$,
where $C(x, x')$ is the set of points with the same label as $x$ which are all closer to $x$ than $x'$ is to $x$. (Note that  $x \in C(x, x')$.)
This constraint enforces that if $x'$ appears in the condensed set (meaning $v(x')=1$), then there must be in the condensed set some point closer to $x$ with the same label as $x$. We also need the constraint 
$
\sum_{x} v(x) \geq 1
$,
to disallow the empty set.
Finally, the objective is to minimize $\sum_x v(x)$,
which corresponds to minimizing the size of the condensed set.
We implented this program using the Python cvxpy library.

\paragraph{Datasets.}
We ran trials on the banana, circle and iris data sets, see Figure \ref{fig:all_data}. 
\begin{itemize}[noitemsep,leftmargin=*] 
    \item Banana. This data set is a synthetic collection, previously used by \citet{FM-21} for NNC. It contains instances arranged in several banana-shaped clusters. (The $x$ and $y$ axes represent the respective properties At1 and At2 defined there.) For our experiment, we retained only 200 of more than 5000 original points. 
    \item Circle. This is a synthetic randomized data set containing 200 points. It contains instances arranged in a circular cluster, surrounded by instances of the opposing class.
    \item Iris. This is the very popular data set of the UCI Machine Learning Repository. It consists of three classes, each containing 50 instances of a certain species of iris. For our experiments, we considered only two classes of the three (namely Setosa and Versicolour). 
\end{itemize}

\begin{figure}
    \centering
    \includegraphics[width=.9\linewidth]{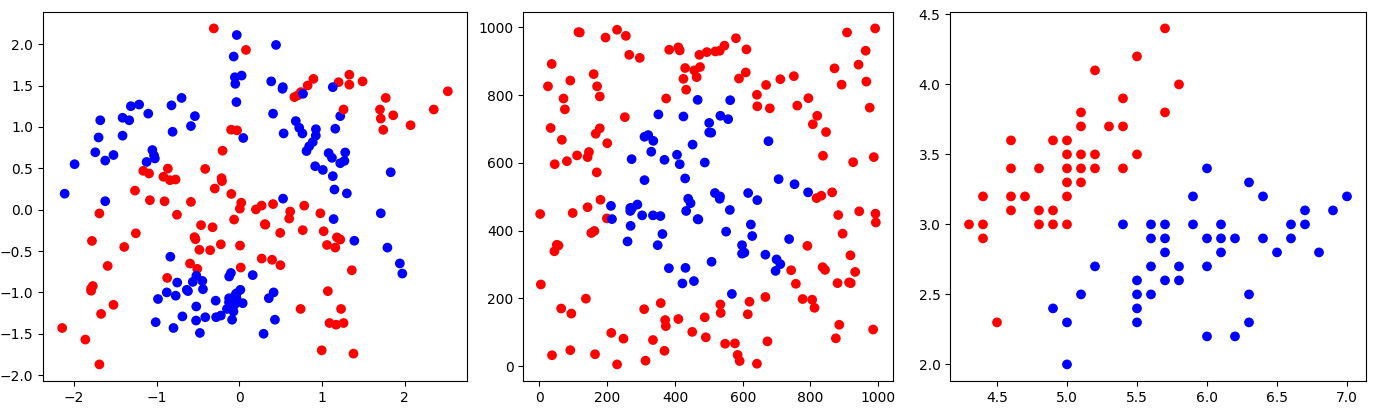}
    \caption{The banana, circle and iris data sets}
    \label{fig:all_data}
\end{figure}

\paragraph{Experiments and results.}
As proof of concept, we ran on the above datasets multiple algorithms for NN condensing -- namely, the MSS \citep{BFS-05} and RSS \citep{FM-21} heuristics, and the exact IP described above -- and also our greedy heuristic for WNN condensing, as presented above in Algorithm \ref{alg:weighted-heuristic}.

We found that our weighted condensing heuristic achieved superior compression when compared to the unweighted heuristics, see Table \ref{tab:condense}. Interestingly, it was still slightly worse that the optimal unweighted solution achieved by the IP.

\begin{table}[ht]
\caption{Comparison of achieved condensing}
\centering\label{result:chosen}
\begin{tabular}{||c c c c c c||} 
 \hline
 Dataset & Points & MSS & RSS & IP & WNN  \\ [0.5ex] 
 \hline\hline
 Circle & 200 & 52 & 45 & 7 & 12  \\ 
 \hline
 Banana & 200 & 74 & 66 & 32 & 35   \\
 \hline
 Iris & 100 & 11 & 9 & 2 & 4   \\
 \hline
\end{tabular}
\label{tab:condense}
\end{table}

%

\subsection{Large dataset trial}
We also investigated the performance of the above heuristics on the notMNIST dataset, published by Yaroslav Bulatov in 2011, and further studied in the context of NN classification by \citet{balsubramani2019adaptive,kerem2023error}. This dataset consists of approximately $19,000$ different font glyphs of the letters A-J, that is ten classes. Each image is of dimension $28\times 28$.
Since this data set is quite large, the IP utilized above is not feasible. Yet for large data sets, we are able to sample from the whole set, and this allows us to report and compare the test error achieved by the classifiers.

To facilitate the execution of our experiment on a standard computer, we first applied uniform manifold approximation and projection dimensionality reduction (UMAP) \citep{mcinnes2018umap}. This served to reduce the dimension from $28\times28$ to only $3$. (See \citet{kerem2023error} for a visual demonstration.)

The dataset was split into train ($70\%$) and test ($30\%$) sets.
We considered the heuristics used in the previous section, that is MSS and RSS for NN condensing, and our greedy heuristic for the WNN rule. As we are also interested in test error, we compared these results to the case where the sample is not compressed (we refer to this as 1-NN).
Figure \ref{fig:notMNIST_results} (top) shows the error obtained on the test set, over 10 realizations of the random split.
The compression ratios achieved are shown in 
Figure \ref{fig:notMNIST_results} (bottom).

These results highlight the following:
RSS was the least competitive in error and compression. While MSS achieved better compression than our greedy WNN, its test error was worse. In particular, the greedy WNN achieved essentially the same error as the standard nearest neighbor rule, while compressing the data-set by approximately $80\%$.

\begin{figure}
    \centering
    \includegraphics[width =.49\columnwidth]{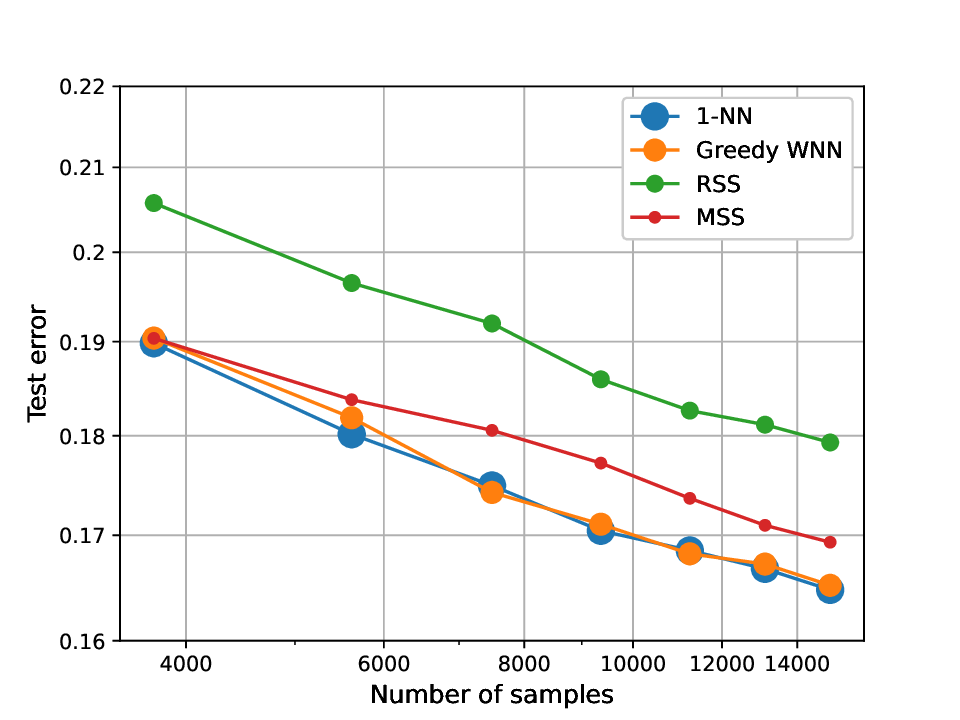}
    \includegraphics[width =.49\columnwidth]{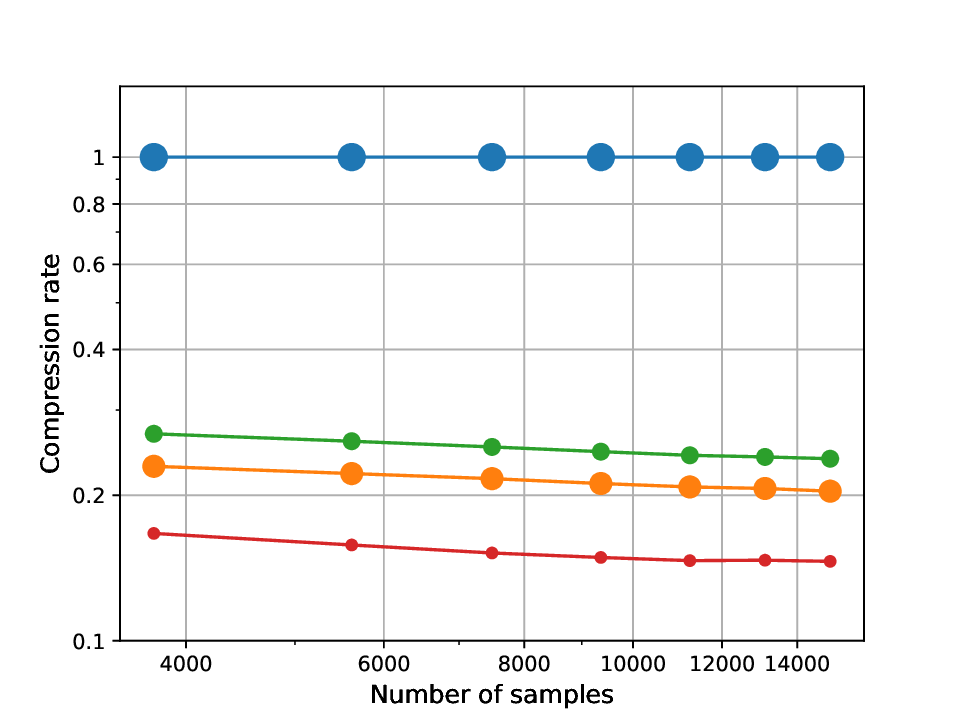}
    \caption{Comparing the test error and compression rates achieved by our greedy WNN heuristic with other popular condensing heuristics.}
    \label{fig:notMNIST_results}
\end{figure}

\section{Conclusions}
We have demonstrated that WNN condensing is more powerful than standard NN condensing, yet is characterized by similar generalization bounds. Hence WNN can only improve the degree of compression, while maintaining the same theoretical guarantees. Our suggested heuristic is theoretically sound, and gave promising empirical results.
This indicates that WNN condensing heuristics are deserving of further study.





\bibliographystyle{plainnat}
\bibliography{refs_imported, refs}

\end{document}